\documentclass[11pt]{article}
\usepackage{amsmath,amssymb,amsthm,amsfonts}
\usepackage{bbm}
\usepackage{graphicx}
\usepackage{color}
\usepackage{appendix}
\usepackage[top=0.5in, bottom=0.8in, left=1in, includeheadfoot]{geometry}
\usepackage{subfigure} 
\usepackage{algorithm}
\usepackage{algorithmic}
\usepackage{color}
\usepackage{authblk}
\RequirePackage[colorlinks]{hyperref}

\theoremstyle{definition}

\theoremstyle{plain}
\newtheorem{thm}{Theorem}
\newtheorem{proposition}[thm]{Proposition}
\newtheorem{lemma}{Lemma}

\newtheorem{corollary}{Corollary}
\newtheorem{assumption}{Assumption}

\newcommand{\ip}[2]{\langle #1,#2 \rangle}
\newcommand{\bydef}{:=}
\newcommand{\expect}[1]{\mathbb{E}\left[{#1}\right]}

\newcommand{\prob}[1]{\mathbb{P}\left[{#1}\right]}

\newcommand{\given}{\; \big\vert \;}
\newcommand{\ignore}[1]{}
\newcommand{\shortver}[1]{}
\newcommand{\appropto}{\mathrel{\vcenter{
  \offinterlineskip\halign{\hfil$##$\cr
    \propto\cr\noalign{\kern2pt}\sim\cr\noalign{\kern-2pt}}}}}

\newcommand{\corr}{\color{black}}

\begin{document} 

\title{Thompson Sampling for Complex Online Problems}

\author{Aditya Gopalan\thanks{Department of Electrical Engineering,
    Technion, Haifa, Israel. Email: {\tt aditya@ee.technion.ac.il}},
  Shie Mannor\thanks{Department of Electrical Engineering, Technion,
    Haifa, Israel. Email: {\tt shie@ee.technion.ac.il}}, Yishay
  Mansour\thanks{Department of Computer Science, Tel Aviv University,
    Tel Aviv, Israel. Email: {\tt mansour@tau.ac.il}}}
\maketitle

\begin{abstract} 
  We consider stochastic multi-armed bandit problems with complex
  actions over a set of basic arms, where the decision maker plays a
  complex action rather than a basic arm in each round.
  The reward of the complex action is some function of the basic arms'
  rewards, and the feedback observed may not necessarily be the reward
  per-arm. For instance, when the complex actions are subsets of the
  arms, we may only observe the maximum reward over the chosen
  subset. Thus, feedback across complex actions may be coupled due to
  the nature of the reward function. 
  {\corr We prove a frequentist regret bound for Thompson sampling in
    a very general setting involving parameter, action and observation
    spaces and a likelihood function over them.} The bound holds for
  discretely-supported priors over the parameter space without
  additional structural properties such as closed-form posteriors,
  conjugate prior structure or independence across arms. The regret
  bound scales logarithmically with time but, more importantly, with
  an improved constant that non-trivially captures the coupling across
  complex actions due to the structure of the rewards. As
  applications, we derive improved regret bounds for classes of
  complex bandit problems involving selecting subsets of arms,
  including the first nontrivial regret bounds for nonlinear MAX
  reward feedback from subsets.
%
\end{abstract}

\section{Introduction}


The stochastic Multi-Armed Bandit (MAB) is a classical framework in
machine learning and optimization.
In the basic MAB setting, there is a finite set of actions, each of
which has a reward derived from some stochastic process, and a learner
selects actions to optimize long-term performance.
The MAB model gives a crystallized abstraction of a fundamental
decision problem -- whether to explore or exploit in the face of
uncertainty. Bandit problems have been extensively studied and several
well-performing methods are known for optimizing the reward
\cite{Gittins,AuerCF02,AudBub04:minimax,GarCap11:klucb}. However, the
requirement that the actions' rewards be independent is often a severe
limitation, as seen in these examples: \\

{\bf Web Advertising:} Consider a website publisher selecting at each
time a subset of ads to be displayed to the user. As the publisher is
paid per click, it would like to maximize its revenue, but
dependencies between different ads could mean that the problem does
not ``decompose nicely''. For instance, showing two car ads might not
significantly increase the click probability over a single car ad. \\

{\bf Job Scheduling:} Assume we have a small number of resources or
machines, and in each time step we receive a set of jobs (the ``basic
arms''), where the duration of each job follows some fixed but unknown
distribution. The latency of a machine is the sum of the latencies of
the jobs (basic arms) assigned to it, and the makespan of the system
is the maximum latency across machines. Here, the decision maker's
complex action is to partition the jobs (basic arms) between the
machines, to achieve the least makespan on
average.\\

{{\bf Routing:} Consider a multi-commodity flow problem, where for
  each source-destination pair, we need to select a route (a complex
  action). In this setting the capacities of the edges (the basic
  arms) are random variables, and the reward is the total flow in the
  system at each time.
In this example, the rewards of different paths are inter-dependent,
since the flow on one path depends on which other paths where
selected.} \\

These examples motivate settings where a model more complex than the
simple MAB is required.
Our high-level goal is to describe a methodology that can tackle
bandit problems with complex action/reward structure, and also
guarantee high performance. A crucial complication in the problems
above is that it is unlikely that we will get to observe the reward of
each basic action chosen. Rather, we can hope to receive only an
aggregate reward for the complex action taken.
Our approach to complex bandit problems stems from the idea that when
faced with uncertainty, pretending to be Bayesian can be
advantageous. A purely Bayesian view of the MAB assumes that the model
parameters (i.e., the arms' distributions) are drawn from a prior
distribution. We argue that even in a frequentist setup, in which the
stochastic model is unknown but fixed, working with a fictitious prior
over the model (i.e., being \emph{pseudo-Bayesian}) helps solve very
general bandit problems with complex actions and observations. \\

Our algorithmic prescription for complex bandits is Thompson sampling
\cite{Thompson,Scott,AgrawalG}: Start with a fictitious prior
distribution over the parameters of the basic arms of the model, whose
posterior gets updated as actions are played. A parameter is randomly
drawn according to the posterior and the (complex) action optimal for
the parameter is played. The rationale behind this is twofold: (1)
Updating the posterior adds useful information about the true unknown
parameter. (2) Correlations among complex bandit actions (due to their
dependence on the basic parameters) are implicitly captured by
posterior updates on the space of basic parameters. \\

The main advantage of a pseudo-Bayesian approach like Thompson
sampling, compared to other MAB methodologies such as UCB, is that it
can handle a wide range of information models that go beyond observing
the individual rewards alone. For example, suppose we observe only the
final makespan in the multi-processor job scheduling problem above. In
Thompson sampling, we merely need to compute a posterior given this
observation and its likelihood. In contrast, it seems difficult to
adapt an algorithm such as UCB for this case without a naive
exponential dependence on the number of basic arms\footnote{The work
  of Dani et al. \cite{DaniHayKak08} first extended the UCB framework
  to the case of linear cost functions. However, for more complex,
  nonlinear rewards (e.g., multi-commodity flows or makespans), it is
  unclear how UCB-like algorithms can be applied other than to treat
  all complex actions independently.}. Besides, the deterministic
approach of optimizing over regions of the parameter space that
UCB-like algorithms follow \cite{DaniHayKak08,AbbPalCsa11:linbandits}
is arguably harder to apply in practice, as opposed to optimizing over
the action space given a sampled parameter in Thompson sampling --
often an efficient polynomial-time routine like a sort. The Bayesian
view that motivates Thompson sampling also allows us to use efficient
numerical algorithms such as particle filtering
\cite{RisticParticleBook04,DoucetSMC01} to approximate complicated
posterior distributions in practice. \\

Our main analytical result is a general regret bound for Thompson
sampling in complex bandit settings. No specific structure is imposed
on the initial (fictitious) prior, except that it be discretely
supported and put nonzero mass on the true model. The bound for this
general setting scales logarithmically with time\footnote{More
  precisely, we obtain a bound of the form $\mathsf{B} + \mathsf{C}
  \log T$, in which $\mathsf{C}$ is a non-trivial preconstant that
  captures precisely the structure of correlations among actions and
  thus is often better than the decoupled
  sum-of-inverse-KL-divergences bounds seen in literature
  \cite{LaiRob:1985}. The additive constant (wrt time) $\mathsf{B}$,
  though potentially large and depending on the total number of
  complex actions, appears to be merely an artifact of our proof
  technique tailored towards extracting the time scaling
  $\mathsf{C}$. This is borne out, for instance, from numerical
  experiments. We remark that such additive constants, in fact, often
  appear in regret analyses of basic Thompson sampling
  \cite{KauKorMun12:thompson,AgrawalG}.}, as is well-known. But more
interestingly, the preconstant for this logarithmic scaling can be
explicitly characterized in terms of the bandit's KL divergence
geometry and represents the {\em information complexity} of the bandit
problem. The standard MAB imposes no structure among the actions, thus
its information complexity simply becomes a sum of terms, one for each
separate action. However, in a complex bandit setting, rewards are
often more informative about other parameters of the model, in which
case the bound reflects the resulting coupling across complex actions. \\

Recent work has shown the regret-optimality of Thompson sampling for
the basic MAB \cite{AgrawalG,KauKorMun12:thompson}, and has even
provided regret bounds for a specific complex bandit setting -- the
linear bandit case where the reward is a linear function of the
actions \cite{AgrGoy13:contextual}. However, the analysis of complex
bandits in general poses challenges that cannot be overcome using the
specialized techniques in these works. Indeed, these existing analyses
rely crucially on the conjugacy of the prior and posterior
distributions -- either independent Beta or exponential family
distributions for basic MAB or standard normal distributions for
linear bandits. These methods break down when analyzing the evolution
of complicated posterior distributions which often lack even a closed
form expression. \\

In contrast to existing regret analyses, we develop a novel proof
technique based on looking at the form of the Bayes posterior. This
allows us to track the posterior distributions that result from
general action and feedback sets, and to express the concentration of
the posterior as a constrained optimization problem in path space. It
is rather surprising that, with almost no specific structural
assumptions on the prior, our technique yields a regret bound that
reduces to Lai and Robbins' classic lower bound for standard MAB, and
also gives non-trivial and improved regret scalings for complex
bandits. In this vein, our results represent a generalization of
existing performance results for Thompson sampling. \\

We have complemented our theoretical findings with numerical studies
of Thompson sampling. The algorithm is implemented using a simple
particle filter \cite{RisticParticleBook04} to maintain and sample
from posterior distributions. We evaluated the performance of the
algorithm on two complex bandit scenarios -- subset selection from a
bandit and job scheduling. \\

\noindent{\bf Related Work:} Bayesian ideas for the multi-armed bandit
date back nearly 80 years ago to the work of W. R. Thompson
\cite{Thompson}, who introduced an elegant algorithm based on
posterior sampling. However, there has been relatively meager work on
using Thompson sampling in the control setup. A notable exception is
\cite{OrtegaBraun10} that develops general Bayesian control rules and
demonstrates them for classic bandits and Markov decision processes
(i.e., reinforcement learning). On the empirical side, a few recent
works have demonstrated the success of Thompson sampling
\cite{Scott,OL11}. Recent work has shown frequentist-style regret
bounds for Thompson sampling in the standard bandit model
\cite{AgrawalG,KauKorMun12:thompson,KorKauMun13:tsexpfam}, and Bayes
risk bounds in the purely Bayesian setting
\cite{OsbRusRoy13:postsamp}. Our work differs from this literature in
that we go beyond simple, decoupled actions/observations -- we focus
on the performance of Thompson setting in a general action/feedback
model, and show novel frequentist regret bounds that account for the
structure of complex actions. \\

Regarding bandit problems with actions/rewards more complex than the
basic MAB, a line of work that deserves particular mention is that of
linear bandit optimization
\cite{Auer03,DaniHayKak08,AbbPalCsa11:linbandits}. In this setting,
actions are identified with decision vectors in a Euclidean space, and
the obtained rewards are random linear functions of actions, drawn
from an unknown distribution. Here, we typically see regret bounds for
generalizations of the UCB algorithm that show polylogarithmic regret
for this setting. However, the methods and bounds are highly tailored
to the specific linear feedback structure and do not carry over to
other kinds of feedback. 


\section{Setup and Notation}
We consider a general stochastic model $X_1, X_2, ...$ of independent
and identically distributed random variables {\corr living in a space
$\mathcal{X}$ (e.g., $\mathcal{X} = \mathbb{R}^{\mathsf{N}}$ if
there is an underlying $\mathsf{N}$-armed basic bandit -- we will
revisit this in detail in Section \ref{subsec:application}).} The
distribution of each $X_t$ is parametrized by $\theta^* \in \Theta$,
where $\Theta$ denotes the parameter space. At each time $t$, an
action $A_t$ is played from an action set $\mathcal{A}$, following
which the decision maker obtains a stochastic observation $Y_t =
f(X_t,A_t) \in \mathcal{Y}$, the observation space, and a scalar
reward $g(f(X_t,A_t))$. Here, $f$ and $g$ are general fixed functions,
and we will often denote $g \circ f$ by the function\footnote{e.g.,
  when $A_t$ is a subset of basic arms, $h(X_t,A_t)$ could denote the
  maximum reward from the subset of coordinates of $X_t$ corresponding
  to $A_t$.} $h$. We denote by $l(y;a,\theta)$ the likelihood of
observing $y$ upon playing action $a$ when the distribution parameter
is $\theta$, i.e.,\footnote{{\corr Finiteness of $\mathcal{Y}$ is
    implicitly assumed for the sake of clarity. In general, when
    $\mathcal{Y}$ is a Borel subset of $\mathbb{R}^\mathsf{N}$,
    $l(\cdot;a,\theta)$ will be the corresponding
    $\mathsf{N}$-dimensional density, etc.}} $l(y;a,\theta) \bydef
\mathbb{P}_\theta[f(X_1,a) = y]$. \\

For $\theta \in \Theta$, let $a^*(\theta)$ be an action that yields
the highest expected reward for a model with parameter $\theta$, i.e.,
$a^*(\theta) \bydef \arg\max_{a \in \mathcal{A}} \mathbb{E}_\theta[
  h(X_1,a)]$.\footnote{The absence of a
  subscript is to be understood as working with the parameter
  $\theta^*$.}. We use $e^{(j)}$ to denote the $j$-th unit vector in
finite-dimensional Euclidean space. \\

The goal is to play an action at each time $t$ to minimize the
(expected) \emph{regret} over $T$ rounds: $R_T \bydef \sum_{t=1}^T
h(X_t,a^*(\theta^*)) - h(X_t,A_t)$, or alternatively, the number of
plays of suboptimal actions\footnote{We refer to the latter objective
  as regret since, under bounded rewards, both the objectives scale
  similarly with the problem size.}: $\sum_{t=1}^T
\mathbf{1}\{A_t \neq a^*\}$. \\

{\corr {\em Remark:} Our main result holds in a more abstract
  stochastic bandit model $(\Theta, \mathcal{Y}, \mathcal{A}, l,
  \hat{h})$ without the need for the underlying ``basic arms''
  $\{X_i\}_i$ and the basic ambient space $\mathcal{X}$. In this case
  we require $l(y;a,\theta) \bydef \mathbb{P}_\theta[Y_1 = y \vert A_1
    = a]$, $\hat{h}:\mathcal{Y} \to \mathbb{R}$ (the reward function),
  $a^*(\theta) \bydef \arg \max_{a \in \mathcal{A}}
  \mathbb{E}_\theta[\hat{h}(Y_1) \vert A_1 = a]$, and the regret $R_T
  \bydef T \hat{h}(Y_0) - \sum_{t=1}^T \hat{h}(Y_t)$ where
  $\mathbb{P}[Y_0 = \cdot] = l(\cdot;a^*(\theta^*),\theta^*)$.} \\

For each action $a \in \mathcal{A}$, define $S_a \bydef \{\theta \in
{\Theta}: a^*(\theta) = a \}$ to be the \emph{decision region} of $a$,
i.e., the set of models in $\Theta$ whose optimal action is
$a$. Within $S_a$, let $S_a'$ be the models that \emph{exactly match}
$\theta^*$ in the sense of the marginal distribution of action $a^*$,
i.e., $S_a' \bydef \{\theta \in S_a: D(\theta^*_{a^*}||\theta_{a^*}) =
0 \}$. Let $S_a''$ be the remaining models in $S_a$.


\begin{algorithm}[tbp]
  \caption{Thompson Sampling}
  \label{alg:ts}
  {\bf Input:} Parameter space $\Theta$, action space $\mathcal{A}$, output
  space $\mathcal{Y}$, likelihood $l(y;a,\theta)$. 

  {\bf Parameter:} Distribution $\pi$ over $\Theta$.
  
  {\bf Initialization:}  Set $\pi_0 = \pi$. \\

  {\bf for} each $t = 1, 2, \ldots$
  \begin{enumerate}
  \item Draw $\theta_t \in \Theta$ according to the distribution
    $\pi_{t-1}$. 

  \item Play $A_t = a^*(\theta_t) \bydef \arg\max_{a \in \mathcal{A}}
    \mathbb{E}_{\theta_t}[ h(X_1,a)]$.

  \item Observe $Y_t = f(X_t,A_t)$.
  
  \item (Posterior Update) Set the distribution $\pi_t$ over $\Theta$ to 
    \[\forall S \subseteq \Theta: \quad \pi_t(S) = \frac{\int_S l(Y_t;A_t,\theta) \pi_{t-1}(d\theta)}{\int_\Theta l(Y_t;A_t,\theta) \pi_{t-1}(d\theta)}. \]
  \end{enumerate}
  {\bf end for}
\end{algorithm}

\section{Regret Performance: Overview}
\label{sec:heuristic}
We propose using Thompson sampling (Algorithm \ref{alg:ts}) to play
actions in the general bandit model. Before formally stating the
regret bound, we present an intuitive explanation of how Thompson
sampling learn to play good actions in a general setup where
observations, parameters and actions are related via a general
likelihood. To this end, let us assume that there are finitely many
actions $\mathcal{A}$. Let us also index the actions in $\mathcal{A}$
as $\{1,2, \ldots, |\mathcal{A}|\}$, with the index $|\mathcal{A}|$
denoting the optimal action $a^*$ (we will require this indexing later
when we associate each coordinate of $|\mathcal{A}|$-dimensional space
with its respective action). Denote by $D(\theta^*_a || \theta_a)$ the
{\em marginal} Kullback-Leibler divergence between the output
distributions of parameters $\theta^*$ and $\theta$ upon playing
action $a$, i.e., between the distributions $\{l(y;a,\theta^*): y \in
\mathcal{Y}\}$ and $\{l(y;a,\theta): y \in \mathcal{Y}\}$. \\

When action $A_t$ is played at time $t$, the prior density gets
updated to the posterior as
$\pi_t(d\theta) \propto \exp\left(-\log
  \frac{l(Y_t;A_t,\theta^*)}{l(Y_t;A_t,\theta)} \right) \pi_{t-1}(d
\theta)$. Observe that the conditional expectation of the
``instantaneous'' log-likelihood ratio $\log
\frac{l(Y_t;A_t,\theta^*)}{l(Y_t;A_t,\theta)}$, is simply the
appropriate marginal KL divergence, i.e.,
\[\expect{\log\frac{l(Y_t;A_t,\theta^*)}{l(Y_t;A_t,\theta)} \given A_t}
= \sum_{a \in \mathcal{A}}\mathbf{1}\{A_t =
a\}D(\theta^*_a||\theta_a).\] Hence, up to a coarse approximation,
\[\log \frac{l(Y_t;A_t,\theta^*)}{l(Y_t;A_t,\theta)} \approx \sum_{a
  \in \mathcal{A}}\mathbf{1}\{A_t = a\}D(\theta^*_a||\theta_a),\] with
which we can write
\begin{equation}
\label{eqn:approximation}
\pi_t(d\theta) \appropto \exp\left(- \sum_{a \in \mathcal{A}} N_t(a)
  D(\theta^*_a||\theta_a)\right) \pi_0(d\theta),
\end{equation} 
with $N_t(a) \bydef \sum_{i=1}^t \mathbf{1}\{A_i = a\}$ denoting the
play count of $a$. The quantity in the exponent can be interpreted as
a ``loss'' suffered by the model $\theta$ up to time $t$, and each time
an action $a$ is played, $\theta$ incurs an additional loss of
essentially the marginal KL divergence $D(\theta^*_a||\theta_a)$. \\

Upon closer inspection, the posterior approximation
(\ref{eqn:approximation}) yields detailed insights into the dynamics
of posterior-based sampling. First, since $\exp\left(- \sum_{a \in
    \mathcal{A}} N_t(a) D(\theta^*_a||\theta_a)\right) \leq 1$, the
true model $\theta^*$ always retains a significant share of posterior
mass: $\pi_t(d\theta^*) \gtrsim \frac{\exp(0) \; \pi_0(d
  \theta^*)}{\int_\Theta 1 \;\pi_0(d \theta)} = \pi_0(d\theta^*)$.
This means that Thompson sampling samples $\theta^*$, and hence plays
$a^*$, with at least a constant probability each time, so that
$N_t(a^*) = \Omega(t)$. \\

Suppose we can show that each model in any $S_a''$, $a \neq a^*$, is
such that $D(\theta^*_{a^*}||\theta_{a^*})$ is bounded strictly away
from $0$ with a gap of $\xi > 0$. Then, our preceding calculation
immediately tells us that any such model is sampled at time $t$ with a
probability exponentially decaying in $t$: $\pi_t(d\theta) \lesssim
\frac{e^{-\xi \Omega(t)}\pi_0(d\theta) }{\pi_0 (d\theta^*)}$; the
regret from such $S_a''$-sampling is \emph{negligible}. On the other
hand, how much does the algorithm have to work to make models in
$S_a'$, $a \neq a^*$ suffer \emph{large} ($\approx \log T$) losses and
thus rid them of significant posterior probability?\footnote{Note:
  Plays of $a^*$ do \emph{not} help increase the losses of these
  models.} \\

A model $\theta \in S_a'$ suffers loss whenever the algorithm plays an
action $a$ for which $D(\theta^*_a || \theta_a) > 0$. Hence, several
actions can help in making a bad model (or set of models) suffer large
enough loss. Imagine that we track the play count vector $N_t \bydef
(N_t(a))_{a \in \mathcal{A}}$ in the integer lattice from $t = 0$
through $t = T$, from its initial value $N_0 = (0,\ldots,0)$. There
comes a first time $\tau_1$ when some action $a_1 \neq a^*$ is
eliminated (i.e., when all its models' losses exceed $\log T$). The
argument of the preceding paragraph indicates that the play count of
$a_1$ will stay fixed at $N_{\tau_1}(a_1)$ for the remainder of the
horizon up to $T$. Moving on, there arrives a time $\tau_2 \geq
\tau_1$ when another action $a_2 \notin \{a*,a_1\}$ is eliminated, at
which point its play count ceases to increase beyond
$N_{\tau_2}(a_2)$, and so on. \\

To sum up: Continuing until all actions $a \neq a^*$ (i.e., the
regions $S_a'$) are eliminated, we have a path-based bound for the
total number of times suboptimal actions can be played. If we let $z_k
= N_{\tau_k}$, i.e., the play counts of all actions at time $\tau_k$,
then for all $i \geq k$ we must have the constraint $z_i(a_k) =
z_k(a_k)$ as plays of $a_k$ do not occur after time
$\tau_k$. Moreover, $\min_{\theta \in S_{a_k}'} \ip{z_k}{D_\theta}
\approx \log T$: action $a_k$ is eliminated precisely at time
$\tau_k$. A bound on the total number of bad plays thus becomes
  \begin{equation}
    \label{eqn:intuitbd}
    \begin{aligned}
      &\max && ||z_k||_1\\
      &\text{ s.t.}
      && \exists \; \mbox{play count sequence} \; \{z_k\}, \\
      &&& \exists \; \mbox{suboptimal action sequence} \; \{a_k\}, \\
      &&& z_i(a_k) = z_k(a_k), i \geq k, \\
      &&& \min_{\theta \in S_{a_k}'} \quad \ip{z_k}{D_\theta} \approx
      \log T, \quad \forall k.
    \end{aligned}
  \end{equation}
  The final constraint above ensures that an action $a_k$ is
  eliminated at time $\tau_k$, and the penultimate constraint encodes
  the fact that the eliminated action $a_k$ is not played after time
  $\tau_k$. The bound not only depends on $\log T$ but also on the
  KL-divergence geometry of the bandit, i.e., the marginal divergences
  $D(\theta^*_a||\theta_a)$. Notice that no specific form for the
  prior or posterior was assumed to derive the bound, save the fact
  that $\pi_0(d\theta^*) \gtrsim 0$, i.e., that the prior puts
  ``enough'' mass on the truth. \\
  
  In fact, all our approximate calculations leading up to the bound
  (\ref{eqn:intuitbd}) hold rigorously -- Theorem \ref{thm:tsregret},
  to follow, states that under reasonable conditions on the prior, the
  number of suboptimal plays/regret scales as (\ref{eqn:intuitbd})
  with high probability. We will also see that the general bound
  (\ref{eqn:intuitbd}) is non-trivial in that (a) for the standard
  multi-armed bandit, it gives essentially the optimum known regret
  scaling, and (b) for a family of complex bandit problems, it can be
  significantly less than the one obtained by treating all actions
  separately.

\section{Regret Performance: Formal Results}
Our main result is a high-probability large-horizon regret
bound\footnote{More precisely, we bound the number of plays of
  suboptimal actions. A bound on the standard regret can also be
  obtained easily from this, via a self-normalizing concentration
  inequality we use in this paper (Appendix
  \ref{app:tsregret}). However, we avoid stating this in the interest
  of minimizing clutter in the presentation, since there will be
  additional $O(\sqrt{\log T})$ terms in the bound on standard
  regret.} for Thompson sampling. The bound holds under the following
mild assumptions about the parameter space $\Theta$, action space
$|\mathcal{A}|$, observation space $|\mathcal{Y}|$, and the fictitious
prior $\pi$.

\begin{assumption}[Finitely many actions, observations]
\label{ass:finite}
$|\mathcal{A}|, |\mathcal{Y}| < \infty$.
\end{assumption}


\begin{assumption}[Finite prior, ``Grain of truth'']
\label{ass:discreteprior}
The prior distribution $\pi$ is supported over a finitely many
particles: $\Theta =\{\theta_1, \ldots, \theta_L\}$, with $L \in
\mathbb{N}$, $\theta^* \in \Theta$ and $\pi(\theta^*) >
0$. Furthermore, there exists $\Gamma \in (0,1/2)$ such that $ \Gamma
\leq l(y;a,\theta) \leq 1-\Gamma$ $\forall \theta \in {\Theta}, a \in
\mathcal{A}, y \in \mathcal{Y}$.
\end{assumption} 

{\em Remark:} We emphasize that the finiteness assumption on the prior
is made primarily for technical tractability, without compromising the
key learning dynamics of Thompson sampling perform well. In a sense, a
continuous prior can be approximated by a fine enough discrete prior
without affecting the geometric structure of the parameter space. The
core ideas driving our analysis explain why Thompson sampling provably
performs well in very general action-observation settings, and, we
believe, can be made general enough to handle even continuous
priors/posteriors. However, the issues here are primarily
measure-theoretic: much finer control will be required to bound and
track posterior probabilities in the latter case, perhaps requiring
the design of {\em adaptive} neighbourhoods of $\theta^*$ with
sufficiently large posterior probabiity that depend on the evolving
history of the algorithm. It is not clear to us how such regions may
be constructed for obtaining regret guarantees in the case of
continuous priors. We thus defer this highly nontrivial task to future
work.

\begin{assumption}[Unique best action]
\label{ass:uba}
The optimal action in the sense of expected reward is
unique\footnote{This assumption is made only for the sake of
  notational ease, and does not affect the paper's results in any
  significant manner.}, i.e., $\mathbb{E}[ h(X_1,a^*) ] > \max_{a \in
  \mathcal{A}, a \neq a^*} \mathbb{E}[ h(X_1,a) ]$.
\end{assumption}

For each action $a \in \mathcal{A}$, let $S_a \bydef \{\theta \in
{\Theta}: a^*(\theta) = a \}$ be the set of parameters for which
playing $a$ is optimal. For any suboptimal action $a \neq a^*$, let
$S_a' \bydef \{\theta \in S_a: D(\theta^*_{a^*}||\theta_{a^*}) = 0
\}$, $S_a'' \bydef S_a \setminus S_a'$, and $\xi \bydef \inf_{\theta
  \in S_a''} D(\theta^*_{a^*}||\theta_{a^*})$. \\

We now state the regret bound for Thompson sampling for general
stochastic bandits. The bound is a rigorous version of the path-based
bound presented earlier, in Section \ref{sec:heuristic}.
\begin{thm}[General Regret Bound for Thompson Sampling]
\label{thm:tsregret}
Under Assumptions \ref{ass:finite}-\ref{ass:uba}, the following holds
for the Thompson Sampling algorithm. For $\delta, \epsilon \in (0,1)$,
there exists $T^\star \geq 0$ such that for all $T \geq T^\star$, with
probability at least $1-\delta$, $\sum_{a \neq a^*}N_T(a) \leq
\mathsf{B} + \mathsf{C}(\log T)$, where $\mathsf{B} \equiv
\mathsf{B}(\delta,\epsilon,\mathcal{A},\mathcal{Y},\Theta,\pi)$ is a
problem-dependent constant that does not depend on $T$, and
\footnote{$\mathsf{C}(\log T) \equiv
  \mathsf{C}(T,\delta,\epsilon,\mathcal{A},\mathcal{Y},\Theta,\pi)$ in
  general, but we suppress the dependence on the problem parameters
  $\delta,\epsilon,\mathcal{A},\mathcal{Y},\Theta,\pi$ as we are
  chiefly concerned with the time scaling.}:
  \begin{equation}
    \label{eqn:tsregret}
    \begin{aligned}
      &&& \mathsf{C}(\log T) \bydef\\ 
      &\max && \sum_{k=1}^{|\mathcal{A}|-1} z_k(a_k)\\
      &\text{ s.t.}
      && z_k \in \mathbb{Z}_+^{|\mathcal{A}|-1} \times \{0\}, a_k \in \mathcal{A} \setminus \{a^*\}, k < |\mathcal{A}|, \\
      &&& z_i \succeq z_k, z_i(a_k) = z_k(a_k), i \geq k, \\
      &&& \forall 1 \leq j,k \leq |\mathcal{A}|-1:\\
      &&& \quad \min_{\theta \in S_{a_k}'} \quad \ip{z_k}{D_\theta} \geq \frac{1+\epsilon}{1-\epsilon}\log T, \\
      &&& \quad \min_{\theta \in S_{a_k}'} \quad \ip{z_k -
        e^{(j)}}{D_\theta} < \frac{1+\epsilon}{1-\epsilon}\log T.
    \end{aligned}
  \end{equation}
\end{thm}
The proof is in Appendix \ref{app:tsregret} of the supplementary
material, and uses a {\corr recently developed} self-normalized
concentration inequality \cite{AbbPalCsa11:linbandits} to help track
the {\corr sample path} evolution of the posterior distribution in its
general form. The power of Theorem \ref{thm:tsregret} lies in the fact
that it accounts for coupling of information across complex actions
and give improved structural constants for the regret scaling than the
standard decoupled case, as we show\footnote{We remark that though the
  non-scaling (with $T$) additive constant $\mathsf{B}$ might appear
  large, we believe it is an artifact of our proof technique tailored
  to extract the time scaling of the regret. Indeed, numerical results
  show practically no additive factor behaviour.} in Corollaries
\ref{cor:corr2} and \ref{cor:corr3}. {\corr We also prove Proposition
  \ref{prop:atleastLmain}, which explicitly quantifies the improvement
  over the naive regret scaling for general complex bandit problems as
  a function of marginal KL-divergence separation in the parameter
  space $\Theta$.}



\subsection{Playing Subsets of Bandit Arms and Observing ``Full
  Information''}
\label{subsec:application}
Let us take a standard $N$-armed Bernoulli bandit with arm parameters
$\mu_1 \leq \mu_2 \leq \cdots \leq \mu_N$. Suppose the (complex)
actions are all size $M$ subsets of the $N$ arms. Following the choice
of a subset, we get to observe the rewards of {\em all} $M$ chosen
arms and receive some bounded reward of the chosen arms (thus,
$\mathcal{Y} = \{0,1\}^M$, $\mathcal{A} = \{S \subset [N]: |S| = M\}$,
$f(\cdot,A)$ is simply the projection onto coordinates of $A \in
\mathcal{A}$, and $g$ is the identity function on $\mathbb{R}^M$). \\

A natural finite prior for this problem can be obtained by
discretizing each of the $N$ basic dimensions and putting uniform mass
over all points: $\Theta = \left\{\beta, 2 \beta, \ldots \left(\lfloor
\frac{1}{\beta}\rfloor - 1\right) \beta \right\}^N$, $\beta \in
(0,1)$, and $\pi(\theta) = \frac{1}{|\Theta|}$ $\forall \theta \in
\Theta$. We can then show, using Theorem \ref{thm:tsregret}, that
\begin{corollary}[Regret for playing subsets of basic arms, Full feedback]
\label{cor:corr2}
Suppose $\mu \equiv (\mu_1, \mu_2, \ldots, \mu_N) \in \Theta$ and
$\mu_{N-M} < \mu_{N-M+1}$. Then, the following holds for the Thompson
sampling algorithm for $\mathcal{Y}$, $\mathcal{A}$, $f$, $g$,
$\Theta$ and $\pi$ as above. For $\delta, \epsilon \in (0,1)$, there
exists $T^\star \geq 0$ such that for all $T \geq T^\star$, with
probability at least $1-\delta$, $\sum_{a \neq a^*}N_T(a) \leq
\mathsf{B}_2 +
\left(\frac{1+\epsilon}{1-\epsilon}\right)\sum_{i=1}^{N-M}
\frac{1}{D(\mu_i || \mu_{N-M+1})} \log T$, where $\mathsf{B}_2 \equiv
\mathsf{B}_2(\delta,\epsilon,\mathcal{A},\mathcal{Y},\Theta,\pi)$ is a
problem-dependent constant that does not depend on $T$.
\end{corollary}

This result, proved in Appendix \ref{app:corr2} of the supplementary
material, illustrates the power of additional information from
observing several arms of a bandit at once. Even though the total
number of actions ${N \choose M}$ is at worst exponential in $M$, the
regret bound scales only as $O((N-M) \log T)$. Note also that for $M =
1$ (the standard MAB setting), the regret scaling is essentially
$\sum_{i=1}^{N-M} \frac{1}{D(\mu_i || \mu_{N-M+1})} \log T$, which is
interestingly the optimal regret scaling for standard Bernoulli
bandits obtained by specialized algorithms for decoupled bandit arms
such as KL-UCB \cite{GarCap11:klucb} and, more recently, Thompson
Sampling with the independent Beta prior \cite{KauKorMun12:thompson}.

\subsection{A General Regret Improvement Result \& Application to MAX Subset Regret}
\label{subsec:applicationmax}
Using the same setting and size-$M$ subset actions as before but
\emph{not} being able to observe all the individual arms' rewards
results in much more challenging bandit settings. Here, we consider
the case where we get to observe as the reward \emph{only} the maximum
value of $M$ chosen arms of a standard $N$-armed Bernoulli bandit. The
feedback is still aggregated across basic arms, but at the same time
very different from the full information case, e.g., observing a
reward of $0$ is very uninformative whereas a value of $1$ is highly
informative about the constituent arms. \\

We can again apply the general machinery provided by Theorem
\ref{thm:tsregret} to obtain a non-trivial regret bound for observing
the highly nonlinear MAX reward. Along the way, we derive the
following consequence of Theorem \ref{thm:tsregret}, useful in its own
right, that explicitly guarantees an improvement in regret directly
based on the Kullback-Leibler resolvability of parameters in the
parameter space -- a measure of coupling across complex actions.

\begin{proposition}[Explicit Regret Improvement Based on Marginal
  KL-divergences in $\Theta$]
  \label{prop:atleastLmain}
  Let $T$ be large enough such that $\max_{\theta \in \Theta, a \in
    \mathcal{A}} D(\theta^*_a || \theta_a) \leq
  \frac{1+\epsilon}{1-\epsilon}\log T$. Suppose $\Delta \leq \min_{a
    \neq a^*, \theta \in S_a'} D(\theta^*_a || \theta_a)$, and that
  the integer $L$ is such that for every $a \neq a^*$ and $\theta \in
  S_a'$, $|\{\hat{a} \in \mathcal{A}: \hat{a} \neq a^*,
  D(\theta^*_{\hat{a}} || \theta_{\hat{a}}) \geq \Delta \}| \geq L$,
  i.e., at least $L$ coordinates of $D_\theta$ (excluding the
  $|\mathcal{A}|$-th coordinate $a^*$) are at least $\Delta$. Then,
  $\mathsf{C}(\log T) \leq \left(\frac{|\mathcal{A}| -
    L}{\Delta}\right) \frac{2(1+\epsilon)}{1-\epsilon} \log T$.
\end{proposition}
Note that the result assures a non-trivial additive reduction of
$\Omega\left(\frac{L}{\Delta} \log T\right)$ from the naive decoupled
regret, whenever any suboptimal model in $\Theta$ can be resolved
apart from $\theta^*$ by at least $L$ actions in the sense of marginal
KL-divergences of their observations. Its proof is contained in
Appendix \ref{app:corr3} in the supplementary material. \\

Turning to the MAX reward bandit, let $\beta \in (0,1)$, and suppose
that $\Theta = \{1-\beta^R, 1-\beta^{R-1}, \ldots, 1-\beta^2,
1-\beta\}^N$, for positive integers $R$ and $N$. As before, let $\mu
\in \Theta$ denote the basic arms' parameters, and let $\mu_{min}
\bydef \min_{a \in \mathcal{A}} \prod_{i \in a} (1-\mu_i)$, and
$\pi(\theta) = \frac{1}{|\Theta|}$ $\forall \theta \in \Theta$. The
action and observation spaces $\mathcal{A}$ and $\mathcal{Y}$ are the
same as those in Section \ref{subsec:application}, but the feedback
function here is $f(x,a) \bydef \max_{i \in a} x_i$, and $g$ is the
identity on $\mathbb{R}$. An application of our general regret
improvement result (Proposition \ref{prop:atleastLmain}) now gives,
for the highly nonlinear MAX reward function,
\begin{corollary}[Regret for playing subsets of basic arms, MAX feedback]
\label{cor:corr3}
The following holds for the Thompson sampling algorithm for
$\mathcal{Y}$, $\mathcal{A}$, $f$, $g$, $\Theta$ and $\pi$ as
above. For $0 \leq M \leq N$, $M \neq \frac{N}{2}$, $\delta, \epsilon
\in (0,1)$, there exists $T^\star \geq 0$ such that for all $T \geq
T^\star$, with probability at least $1-\delta$, $\sum_{a \neq a^*}
N_T(a) \leq \mathsf{B}_3 + (\log 2)\left(
\frac{1+\epsilon}{1-\epsilon}\right) \left[1+ {N-1 \choose M} \right]
\frac{ \log T}{\mu^2_{\min}(1-\beta)}$.
\end{corollary}
Observe that this regret bound is of the order of ${N-1 \choose M}
\frac{\log T}{\mu^2_{\min}}$, which is significantly {\em less} than
the usual, decoupled bound of $|\mathcal{A}| \frac{\log
  T}{\mu^2_{\min}} = {N \choose M} \frac{\log T}{\mu^2_{\min}}$ by a
multiplicative factor of $\frac{{N-1 \choose M}}{{N \choose M}} =
\frac{N-M}{N}$, or by an additive factor of ${N-1 \choose M-1}
\frac{\log T}{\mu^2_{\min}}$. In fact, though this is a provable
reduction in the regret scaling, the actual reduction is likely to be
much better in practice -- experimental results attest to this. The
proof of this result uses sharp combinatorial estimates relating to
vertices on the $N$-dimensional hypercube \cite{AhlAydKha03:weight},
and can be found in Appendix \ref{app:corr3}, in the supplementary
material.

\section{Discussion \& Future Work}
We applied Thompson sampling to balance exploration and exploitation
in bandit problems where the action/observation space is
complex. Using a novel technique of viewing posterior evolution as a
path-based optimization problem, we developed a generic regret bound
for Thompson sampling with improved constants that capture the
structure of the problem. In practice, the algorithm is easy to
implement using sequential Monte-Carlo methods such as particle
filters. \\

Moving forward, the technique of converting posterior concentration to
an optimization involving exponentiated KL divergences could be useful
in showing {\em adversarial} regret bounds for Bayesian-inspired
algorithms. It is reasonable to posit that Thompson sampling would
work well in a range of complex learning settings where a suitable
point estimate is available. As an example, optimal bidding for online
repeated auctions depending on continuous bid reward functions can be
potentiallly learnt by constructing an estimate of the bid curve. \\

Another unexplored direction is handling large scale reinforcement
learning problems with complex, state-dependent Markovian dynamics. It
would be promising if computationally demanding large-state space MDPs
could be solved using a form of Thompson sampling by policy iteration
after sampling from a parameterized set of MDPs; this has previously
been shown to work well in practice \cite{Poupart10,OrtegaBraun10}. We
can also attempt to develop a theoretical understanding of
pseudo-Bayesian learning for complex spaces like the $X$-armed bandit
problem \cite{SrinivasKKS10,BubeckMSS11} with a continuous state
space. At a fundamental level, this could result in a rigorous
characterization of Thompsn sampling/pseudo-Bayesian procedures in
terms of the {value of information} per learning step.

\bibliographystyle{ieeetr}

\begin{thebibliography}{10}

\bibitem{Gittins}
J.~C. Gittins, K.~D. Glazebrook, and R.~R. Weber, {\em Multi-Armed Bandit
  Allocation Indices}.
\newblock Wiley, 2011.

\bibitem{AuerCF02}
P.~Auer, N.~Cesa-Bianchi, and P.~Fischer, ``Finite-time analysis of the
  multiarmed bandit problem,'' {\em Machine Learning}, vol.~47, no.~2-3,
  pp.~235--256, 2002.

\bibitem{AudBub04:minimax}
J.-Y. Audibert and S.~Bubeck, ``Minimax policies for adversarial and stochastic
  bandits,'' in {\em Proceedings of the 22nd Annual Conference on Learning
  Theory, Omnipress}, pp.~773--818, 2004.

\bibitem{GarCap11:klucb}
A.~Garivier and O.~Capp{\'e}, ``The {KL-UCB} algorithm for bounded stochastic
  bandits and beyond,'' {\em Journal of Machine Learning Research - Proceedings
  Track}, vol.~19, pp.~359--376, 2011.

\bibitem{Thompson}
W.~R. Thompson, ``On the likelihood that one unknown probability exceeds
  another in view of the evidence of two samples,'' {\em Biometrika}, vol.~24,
  no.~3--4, pp.~285--294, 1933.

\bibitem{Scott}
S.~Scott, ``A modern {B}ayesian look at the multi-armed bandit,'' {\em Applied
  Stochastic Models in Business and Industry}, vol.~26, pp.~639--658, 2010.

\bibitem{AgrawalG}
S.~Agrawal and N.~Goyal, ``Analysis of {T}hompson sampling for the multi-armed
  bandit problem.,'' {\em Journal of Machine Learning Research - Proceedings
  Track}, vol.~23, pp.~39.1--39.26, 2012.

\bibitem{DaniHayKak08}
V.~Dani, T.~P. Hayes, and S.~M. Kakade, ``Stochastic linear optimization under
  bandit feedback,'' in {\em COLT}, pp.~355--366, 2008.

\bibitem{AbbPalCsa11:linbandits}
Y.~Abbasi-{Y}adkori, D.~Pal, and C.~Szepesvari, ``Improved algorithms for
  linear stochastic bandits,'' in {\em Advances in Neural Information
  Processing Systems 24}, pp.~2312--2320, 2011.

\bibitem{RisticParticleBook04}
B.~Ristic, S.~Arulampalam, and N.~Gordon, {\em Beyond the Kalman Filter:
  Particle Filters for Tracking Applications}.
\newblock Artech House, 2004.

\bibitem{DoucetSMC01}
A.~Doucet, N.~D. Freitas, and N.~Gordon, {\em Sequential Monte Carlo Methods in
  Practice}.
\newblock Springer, 2001.

\bibitem{LaiRob:1985}
T.~L. Lai and H.~Robbins, ``Asymptotically efficient adaptive allocation
  rules,'' {\em Advances in Applied Mathematics}, vol.~6, no.~1, pp.~4--22,
  1985.

\bibitem{KauKorMun12:thompson}
E.~Kaufmann, N.~Korda, and R.~Munos, ``Thompson sampling: An asymptotically
  optimal finite-time analysis,'' in {\em Proceedings of the Twenty-third
  International Conference on Algorithmic Learning Theory}, 2012.

\bibitem{AgrGoy13:contextual}
S.~Agrawal and N.~Goyal, ``Thompson sampling for contextual bandits with linear
  payoffs,'' in {\em Advances in Neural Information Processing Systems 24},
  pp.~2312--2320, 2011.

\bibitem{OrtegaBraun10}
P.~A. Ortega and D.~A. Braun, ``A minimum relative entropy principle for
  learning and acting,'' {\em JAIR}, vol.~38, pp.~475--511, 2010.

\bibitem{OL11}
O.~Chapelle and L.~Li, ``An empirical evaluation of {Thompson} sampling,'' in
  {\em NIPS-11}, 2011.

\bibitem{KorKauMun13:tsexpfam}
N.~Korda, E.~Kaufmann, and R.~Munos, ``Thompson sampling for 1-dimensional
  exponential family bandits,'' in {\em NIPS}, 2013.

\bibitem{OsbRusRoy13:postsamp}
I.~Osband, D.~Russo, and B.~V. Roy, ``({M}ore) efficient reinforcement learning
  via posterior sampling,'' in {\em NIPS}, 2013.

\bibitem{Auer03}
P.~Auer, ``Using confidence bounds for exploitation-exploration trade-offs,''
  {\em J. Mach. Learn. Res.}, vol.~3, pp.~397--422, 2003.

\bibitem{AhlAydKha03:weight}
R.~Ahlswede, H.~Aydinian, and L.~Khachatrian, ``Maximum number of constant
  weight vertices of the unit n-cube contained in a k-dimensional subspace,''
  {\em Combinatorica}, vol.~23, no.~1, pp.~5--22, 2003.

\bibitem{Poupart10}
P.~Poupart, {\em Encyclopedia of Machine Learning}.
\newblock Springer, 2010.

\bibitem{SrinivasKKS10}
N.~Srinivas, A.~Krause, S.~Kakade, and M.~Seeger, ``Gaussian process
  optimization in the bandit setting: No regret and experimental design,'' in
  {\em ICML}, pp.~1015--1022, 2010.

\bibitem{BubeckMSS11}
S.~Bubeck, R.~Munos, G.~Stoltz, and C.~Szepesv{\'a}ri, ``X-armed bandits,''
  {\em J. Mach. Learn. Res.}, vol.~12, pp.~1655--1695, 2011.

\end{thebibliography}

\newpage
\onecolumn

{\bf \Large Appendices: Thompson Sampling for Complex Online Problems}
\appendix
\section{Proof of Theorem \ref{thm:tsregret}}
\label{app:tsregret}
{\bf Sampling from the posterior as proportional to exponential
  weights: } Let $N_t(a)$ be the number of times action $a$ has been
played up to (and including) time $t$. At any time $t$, the posterior
distribution $\pi_t$ over $\Theta$ is given by Bayes' rule:
\begin{align}
  \label{eqn:wts}
  \quad \forall S \subseteq \Theta: \quad \pi_t(S) &=
  \frac{W_t(S)}{W_t(\Theta)}, \quad W_t(S) \bydef
  \int_S W_t(\theta) \pi(d\theta),
\end{align}
with the weight $W_t(\theta)$ of each $\theta$ being the likelihood of
observing the history under $\theta$:
\begin{align*}
  W_t(\theta) &\bydef \prod_{i=1}^{t} \left[\frac{l(Y_i;A_i,\theta)}{l(Y_i;A_i,\theta^*)}\right] = \prod_{a \in \mathcal{A}} \prod_{y \in \mathcal{Y}} \prod_{i=1}^{t} \left[\frac{l(y;a,\theta)}{l(y;a,\theta^*)}\right]^{\mathbf{1}\{A_i = a, Y_i = y\}} \\
  &= \exp \left( - \sum_{a \in \mathcal{A}} \sum_{y \in \mathcal{Y}} \sum_{i=1}^{t} \mathbf{1}\{A_i = a, Y_i = y\} \log \frac{l(y;a,\theta^*)}{l(y;a,\theta)}  \right)  \\
  &= \exp \left( - \sum_{a \in \mathcal{A}} N_t(a) \sum_{y \in \mathcal{Y}} \frac{\sum_{i=1}^{t} \mathbf{1}\{A_i = a, Y_i = y\}}{N_t(a)} \log \frac{l(y;a,\theta^*)}{l(y;a,\theta)}  \right),
\end{align*} 
where we set $N_t(a) \bydef \sum_{i=1}^t \mathbf{1}\{A_i = a\}$. Let
$Z_t(a,y) \bydef \frac{\sum_{i=1}^t \mathbf{1}\{A_i = a, Y_i =
  y\}}{N_t(a)}$, and $Z_t(a) \bydef (Z_t(a,y))_{y \in \mathcal{Y}} \in
\mathbb{R}^{|\mathcal{Y}|}$. Thus $Z_t(a)$ is the empirical
distribution of the observations from playing action $a$ up to time
$t$. The expression for $W_t(\theta)$ above becomes
\begin{align}
  W_t(\theta) &= \exp \left( -\sum_{a \in \mathcal{A}} N_t(a)
    D(\theta^*_a||\theta_a) - \sum_{a \in \mathcal{A}} N_t(a) \sum_{y
      \in \mathcal{Y}} \left(Z_t(a,y) - l(y;a,\theta^*) \right) \log
    \frac{l(y;a,\theta^*)}{l(y;a,\theta)} \right). \label{eqn:wtpost}
\end{align}
Here, for any $\theta \in \Theta$ and $a \in \mathcal{A}$, $\theta_a$
is used to denote the ``marginal'' probability distribution
$\{l(y;a,\theta)\}_{y \in \mathcal{Y}}$ of the output of action $a$
when the bandit has parameter $\theta$. For probability measures $\nu,
\mu$ over $\mathcal{Y}$, $D(\nu || \mu)$ measures the Kullback-Leibler
(KL) divergence of $\nu$ wrt $\mu$. \\

Note that by definition, $W_t(\theta^*) = 1$ at all times $t$ -- a
fact that we use often in the analysis. \\

Instead of observing $Y_t = f(X_t,A_t)$ at each round $t$, consider
the following alternative probability space for the stochastic bandit
in a time horizon $1, 2, \ldots$ with probability measure
$\tilde{\mathbb{P}}$. First, for each action $a \in \mathcal{A}$ and
each time $k = 1, 2, \ldots$, an independent random variable $Q_a(k)
\in \mathcal{Y}$, is drawn with $\prob{Q_a(k) = y} =
l(y;a,\theta^*)$. Denote by $Q \equiv \{Q_a(k)\}_{a \in \mathcal{A},k
  \geq 1}$ the $|\mathcal{A}| \times \infty$ matrix of these
independent random variables. Next, at each round $t = 1, 2, \ldots$,
playing action $A_t = a$ yields the observation $Y_t =
Q_a(N_a(t)+1)$. Thus, in this space,
\[ Z_t(a,y) = U_{N_t(a)}(a,y), \; \mbox{where} \; U_j(a,y) \bydef
\frac{1}{j}\sum_{k=1}^{j} \mathbf{1}\{Q_a(k) = y\}.\]

The following lemma shows that the distribution of sample paths
\emph{seen by a bandit algorithm} in both probability spaces (i.e.,
associated with the measures $\mathbb{P}$ and $\tilde{\mathbb{P}}$) is
identical. This allows us to equivalently work in the latter space to
make statements about the regret of an algorithm.

\begin{lemma}
  For any action-observation sequence $(a_t,y_t)$, $t = 1, \ldots, T$
  of a bandit algorithm,
  \[\tilde{\mathbb{P}}\left[  \forall 1 \leq t \leq T \; (A_t,Y_t) = (a_t,y_t)\right] = \prob{\forall 1 \leq t \leq T \; (A_t,Y_t) = (a_t,y_t)}. \]
\end{lemma}

Henceforth, we will drop the tilde on $\tilde{\mathbb{P}}$ and always
work in the latter probability space, involving the matrix
$Q$. 

\begin{lemma}
  For any suboptimal action $a \neq a^*$,
  \[ \delta_a = \min_{\theta \in S_a'}
  D(\theta^*_{a}||\theta_{a}) > 0. \]
\end{lemma}

Let $N_t'(a)$ (resp. $N_t''(a)$) be the number of times that a
parameter has been drawn from $S_a'$ (resp. $S_a''$), so that $N_t(a)
= N_t'(a) + N_t''(a)$. \\

The following self-normalized, uniform deviation bound controls the
empirical distribution of each row $Q_a(\cdot)$ of the random reward
matrix $Q$. It is a version of a bound proved in
\cite{AbbPalCsa11:linbandits}.

\begin{thm}
  \label{thm:selfnormalized}
  Let $a \in \mathcal{A}$, $y \in \mathcal{Y}$ and $\delta \in
  (0,1)$. Then, with probability at least $1-{\delta}{\sqrt{2}}$,
  \[ \forall k \geq 1 \quad |U_k(a,y) - l(y;a,\theta^*)| \leq 4 \sqrt{\frac{1}{k}\log{\left(\frac{\sqrt{k}}{\delta}\right)}}. \]
\end{thm}

Put $c := \log\frac{|\mathcal{Y}| |\mathcal{A}|}{\delta}$, and
$\rho(x) \equiv \rho_c(x) := 4\sqrt{c + \frac{\log x}{2}}$ for $x >
0$. It follows that the following ``good data'' event occurs with
probability at least $(1-\delta\sqrt{2})$:
\[ G \equiv G(c) := \left\{\forall a \in \mathcal{A} \; \; \forall y
  \in \mathcal{Y} \; \forall k \geq 1 \; \; |U_k(a,y) - l(y;a,\theta^*)|
  \leq \frac{\rho(k)}{\sqrt{k}} \right\}. \]

\begin{lemma}
  \label{lem:wtub}
  Fix $\epsilon \in (0,1)$. There exist $\lambda, n^\star \geq 0$, not
  depending on $T$, so that the following is true. For any $\theta \in
  \Theta$, $a \in \mathcal{A}$ and $y \in \mathcal{Y}$, under the event
  $G$,
  \begin{enumerate}
  \item At all times $t \geq 1$,
    \[ N_t(a) D(\theta^*_a||\theta_a) + N_t(a) \sum_{y \in \mathcal{Y}}
    \left(Z_t(a,y) - l(y;a,\theta^*) \right) \log
    \frac{l(y;a,\theta^*)}{l(y;a,\theta)} \geq - \lambda,\]
  \item If $N_t(a) \geq n^\star$, then
    \[N_t(a) D(\theta^*_a||\theta_a) + N_t(a) \sum_{y \in \mathcal{Y}}
    \left(Z_t(a,y) - l(y;a,\theta^*) \right) \log
    \frac{l(y;a,\theta^*)}{l(y;a,\theta)} \geq (1-\epsilon)N_t(a)
    D(\theta^*_a||\theta_a).\]
  \end{enumerate}
\end{lemma}
\begin{proof}
  Under $G$, we have
  \begin{align}
    & N_t(a) D(\theta^*_a||\theta_a) + N_t(a) \sum_{y \in \mathcal{Y}}
    \left(Z_t(a,y) - l(y;a,\theta^*) \right) \log
    \frac{l(y;a,\theta^*)}{l(y;a,\theta)} \nonumber \\
    &\geq N_t(a) D(\theta^*_a||\theta_a) - N_t(a) \sum_{y \in \mathcal{Y}}
    |Z_t(a,y) - l(y;a,\theta^*)| \left|\log
      \frac{l(y;a,\theta^*)}{l(y;a,\theta)}\right| \nonumber \\
    &\geq N_t(a) D(\theta^*_a||\theta_a) - \rho(N_t(a)) \sqrt{N_t(a)} 
    \sum_{y \in \mathcal{Y}} \left| \log
      \frac{l(y;a,\theta^*)}{l(y;a,\theta)}\right|. \label{eqn:wtub}
  \end{align}
  
  For a fixed $\theta \in \Theta$, $a \in \mathcal{A}$, the expression
  above diverges to $+\infty$, viewed as a function of $N_t(a)$, as
  $N_t(a) \to \infty$ (except when $\theta_a = \theta^*_a$, in which
  case the expression is identically $0$.)  Hence, the expression
  achieves a finite minimum $-\lambda_{\theta,a}$ (not depending on
  $T$) over non-negative integers $N_t(a) \in \mathbb{Z}^+$. Since
  there are only finitely many parameters $\theta \in \Theta$, it
  follows that if we set $\lambda := \max_{\theta \in \Theta, a \in
    \mathcal{A}} \lambda_{\theta,a}$, then the above expression is
  bounded below by $-\lambda$, uniformly across $\Theta$. This proves
  the first part of the lemma. \\

  To show the second part, notice again that for fixed $\theta \in
  \Theta$ and $a \in \mathcal{A}$, there exists $n^\star_{\theta,a}
  \geq 0$ such that
  \[\rho(x) \sqrt{x} \sum_{y \in \mathcal{Y}}
  \left| \log \frac{l(y;a,\theta^*)}{l(y;a,\theta)}\right| \leq
  \epsilon x D(\theta^*_a||\theta_a), \quad x \geq
  n^\star_{\theta,a} \] since $\rho(x) = o(x)$. Setting $n^\star :=
  \max_{\theta \in \Theta, a \in \mathcal{A}} n^\star_{\theta,a}$
  then completes the proof of the second part. 
\end{proof}

\subsection{Regret due to sampling from $S_a''$}
The result of Lemma \ref{lem:wtub} implies that under the event $G$,
and at all times $t \geq 1$:
\begin{align}
  \pi_t(\theta^*) &= \frac{W_t(\theta^*)\pi(\theta^*)}{\int_{\Theta}
    W_t(\theta) \pi(d\theta)} = \frac{\pi(\theta^*)}{\int_{\Theta}
    W_t(\theta)
    \pi(d\theta)} \nonumber \\
  &\geq \frac{\pi(\theta^*)}{\int_{\Theta} \exp\left(
      \lambda|\mathcal{A}| \right) \pi(d\theta)} =
  \pi(\theta^*)e^{-\lambda|\mathcal{A}|} \equiv p^*, \;
  \mbox{say}. \label{eqn:probthetastG}
\end{align}

Also, under the event $G$, the posterior probability of $\theta \in
S_a''$ at all times $t$ can be bounded above using Lemma
\ref{lem:wtub} and the basic bound in (\ref{eqn:wtub}):
\begin{align*}
  &\pi_t(\theta) = \frac{W_t(\theta)\pi(\theta)}{\int_{\Theta}
    W_t(\psi) \pi(d\psi)} \leq
  \frac{W_t(\theta)\pi(\theta)}{\pi(\theta^*)} \\
  &= \frac{\pi(\theta)}{\pi(\theta^*)} \exp \left( -\sum_{a \in
      \mathcal{A}} N_t(a) D(\theta^*_a||\theta_a) - \sum_{a \in
      \mathcal{A}} N_t(a) \sum_{y \in \mathcal{Y}} \left(Z_t(a,y) -
      l(y;a,\theta^*) \right) \log
    \frac{l(y;a,\theta^*)}{l(y;a,\theta)}  \right) \\
  &\leq \frac{\pi(\theta) e^{\lambda |\mathcal{A}|}}{\pi(\theta^*)} \exp \left( - N_t(a^*)
    D(\theta^*_{a^*}||\theta_{a^*}) - N_t(a^*) \sum_{y \in
      \mathcal{Y}} \left(Z_t(a^*,y) - l(y;a^*,\theta^*) \right) \log
    \frac{l(y;a^*,\theta^*)}{l(y;a^*,\theta)} \right) \\
  &\leq \frac{\pi(\theta)e^{\lambda |\mathcal{A}|}}{\pi(\theta^*)} \exp \left( - N_t(a^*)
    D(\theta^*_{a^*}||\theta_{a^*}) + \rho(N_t(a)) \sqrt{N_t(a^*)} \sum_{y \in
      \mathcal{Y}}\left| \log
    \frac{l(y;a^*,\theta^*)}{l(y;a^*,\theta)}\right| \right).
\end{align*}
In the above, the penultimate inequality is by Lemma \ref{lem:wtub}
applied to all actions $a \neq a^*$, and the final inequality follows
in a manner similar to (\ref{eqn:wtub}), for action $a^*$. Letting $d
:= \frac{e^{\lambda |\mathcal{A}|}}{\pi(\theta^*)}$, we have that
under the event $G$, for $a \neq a^*$ and $\theta \in S_a''$,
\begin{align}
  \pi_t(\theta) &\leq d \pi(\theta) \exp \left( - N_t(a^*)
    D(\theta^*_{a^*}||\theta_{a^*}) + \rho(N_t(a))\sqrt{N_t(a^*)} \sum_{y \in
      \mathcal{Y}}\left| \log
    \frac{l(y;a^*,\theta^*)}{l(y;a^*,\theta)}\right| \right).  \label{eqn:badwtbd}
\end{align}

Recall that by definition, any $\theta \in S_a''$ with $a \neq a^*$
can be resolved apart from $\theta^*$ in the action $a^*$, i.e.,
$D(\theta^*_{a^*}||\theta_{a^*}) \geq \xi$. Moreover, the
discrete prior assumption (Assumption \ref{ass:discreteprior}) implies
that $\xi > 0$. Using this, we can bound the right-hand side of
(\ref{eqn:badwtbd}) further under the event $G$:
\begin{align}
  \pi_t(\theta) &\leq d \pi(\theta) \exp \left( - \xi N_t(a^*) +
    2\rho(N_t(a))\sqrt{N_t(a^*)} \log \frac{1-\Gamma}{\Gamma}
  \right). \label{eqn:thetasa''}
\end{align}
Integrating (\ref{eqn:thetasa''}) over $\theta \in S_a''$ and noticing
that $\pi(S_a'') \leq 1$ gives, under $G$,
\begin{equation} 
  \pi_t(S_a'') \leq d \exp \left( - \xi N_t(a^*) + 2\rho(N_t(a))\sqrt{N_t(a^*)} \log
    \frac{1-\Gamma}{\Gamma} \right).  \label{eqn:sa''weight} 
\end{equation}

We can now estimate, using the conditional version of Markov's
inequality, the number of times that parameters from $S_a''$ are
sampled under ``good data'' $G$:
\begin{align}
  &\prob{\sum_{t=1}^\infty \mathbf{1}\{\theta_t \in S_a''\} > \eta \given
    G} \leq \eta^{-1}\sum_{t=1}^\infty\expect{\mathbf{1}\{\theta_t \in
    S_a''\} \given G} = \eta^{-1}\sum_{t=1}^\infty\expect{\pi_t(S_a'') \given G} \nonumber \\
  &\leq \eta^{-1} \sum_{t=1}^\infty \left( 1 \wedge \expect{ d \exp \left( - \xi N_t(a^*) + 2\rho(N_t(a))\sqrt{N_t(a^*)} \log
    \frac{1-\Gamma}{\Gamma} \right) \given G}\right), \label{eqn:sa''bound}
\end{align}
where the final inequality is by (\ref{eqn:sa''weight}) and the fact
that $\pi_t(S_a'') \leq 1$.\footnote{$a \wedge b$ denotes the minimum
  of $a$ and $b$.}

At each time $t$, if we let $\mathcal{F}_t$ denote the
$\sigma$-algebra generated by the random variables
$\{(\theta_i,A_i,Y_i): i \leq t\}$, then
\begin{align*}
  \expect{e^{- \xi N_t(a^*)} \given G } &= \expect{\expect{e^{- \xi N_t(a^*)} \given \mathcal{F}_{t-1},G } \given G } \\
  &= \expect{ e^{- \xi N_{t-1}(a^*)} \expect{e^{- \xi \mathbf{1}\{A_t = a^* \} } \given \mathcal{F}_{t-1},G } \given G } \\
  &\leq \expect{ e^{- \xi N_{t-1}(a^*)} \expect{e^{- \xi \mathbf{1}\{\theta_t = \theta^* \} } \given \mathcal{F}_{t-1},G } \given G } \\
  &\quad \mbox{($\theta_t = \theta \Rightarrow A_t = a^*$)} \\
  &= \expect{ e^{- \xi N_{t-1}(a^*)} \left(\pi_t(\theta^*) e^{-\xi} + 1-\pi_t(\theta^*) \right)    \given G } \\
  &\leq \expect{ e^{- \xi N_{t-1}(a^*)} \left(p^* e^{-\xi} +
      1-p^* \right) \given G } \\
  &= \left(p^* e^{-\xi} + 1-p^* \right) \expect{ e^{- \xi
      N_{t-1}(a^*)} \given G },
\end{align*}
where, in the penultimate step, we use $\pi_t(\theta^*) \geq p^* \cdot
\mathbf{1}_G$ from (\ref{eqn:probthetastG}). Iterating this estimate
and using it in (\ref{eqn:sa''bound}) together with the trivial bound
$\sqrt{N_t(a^*)} \leq \sqrt{t}$ gives
\begin{align*}
  \prob{\sum_{t=1}^\infty \mathbf{1}\{\theta_t \in S_a''\} > \eta \given
    G} &\leq \eta^{-1} \sum_{t=1}^\infty \left( 1 \wedge  d\left(p^* e^{-\xi} + 1-p^*
  \right)^t \exp \left(2 \rho(t)\sqrt{t} \log\frac{1-\Gamma}{\Gamma} \right)  \right).
\end{align*}
Since $p^* e^{-\xi} + 1-p^* < 1$ and $\rho(t)\sqrt{t} = o(t)$,
the sum above is dominated by a geometric series after finitely many
$t$, and is thus a finite quantity $\alpha < \infty$, say. (Note that
$\alpha$ does not depend on $T$.) Replacing $\delta$ by
$\frac{\delta}{|\mathcal{A}|}$ and taking a union bound over all $a
\neq a^*$, this proves
\begin{lemma}
  \label{lem:notmuchSa''}
  There exists $\alpha < \infty$ such that
  \[\prob{G, \exists a \neq a^* \; \sum_{t=1}^\infty \mathbf{1}\{\theta_t \in S_a''\}
  > \frac{\alpha |\mathcal{A}|}{\delta}} \leq \delta. \]
\end{lemma}

\subsection{Regret due to sampling from $S_a'$}
For $\theta \in \Theta$, $a \in \mathcal{A}$, define
$b_{\theta,a}:\mathbb{R}^+ \to \mathbb{R}$ by
\begin{equation*}
  b_{\theta,a}(x) := \left\{
    \begin{array}{cc}
      -\lambda, & x < n^\star \\
      (1-\epsilon) x D(\theta^*_a || \theta_a), & x \geq n^\star,
    \end{array}
 \right.
\end{equation*}
where $\lambda$ and $n^\star$ satisfy the assertion of Lemma
\ref{lem:wtub}. Thus, by Lemma \ref{lem:wtub}, under $G$, and for all
$\theta \in \Theta$,
\begin{equation*}
  W_t(\theta) \leq e^{- \sum_{a \in \mathcal{A}} b_{\theta,a}(N_t(a))} \leq e^{- \sum_{a \in \mathcal{A}} b_{\theta,a}(N_t'(a))},
\end{equation*}
where the last inequality is because $N_t(a) = N_t'(a) + N_t''(a)$,
and because $b_{\theta,a}(x)$ is monotone non-decreasing in $x$. \\

{\em Note:} In what follows, we assume that $T > 0$ is large enough
such that $\log T \geq \frac{\lambda |\mathcal{A}|}{\epsilon}$ holds. \\

We proceed to define the following sequence of non-decreasing stopping
times, and associated sets of actions, for the time horizon $1,2,
\ldots, T$.

Let $\tau_0 \bydef 1$ and $\mathcal{A}_0 \bydef \emptyset$. For each
$k = 1, \ldots, |\mathcal{A}|-1$, let
\begin{equation}
\label{eqn:deftau}
\begin{aligned}
  \tau_k \bydef \; &\min && \tau_{k-1} \leq t \leq T \\
  &\text{ s.t.}
  && \mathsf{a}_k \notin \mathcal{A}_{k-1} \cup \{a^*\}, \\
  &&& \min_{\theta \in S_{\mathsf{a}_k}'}  \sum_{m=1}^{k-1} N_{\tau_m}'(a_m)D(\theta^*_{a_m} || \theta_{a_m}) + 
  \sum_{\substack{a \notin \mathcal{A}_{k-1}}} N_t'(a) D(\theta^*_a || \theta_a) \geq \frac{1+\epsilon}{1-\epsilon} \log T.
\end{aligned}
\end{equation}

In other words, for each $k$, $\mathcal{A}_k$ represents a set of
``eliminated'' suboptimal actions. $\tau_k$ is the first time after
$\tau_{k-1}$, when some suboptimal action (which is not already
eliminated) gets eliminated in the sense of satisfying the inequality
in (\ref{eqn:deftau}). Essentially, the inequality checks whether the
condition
\[ \sum_{a \neq a^*} N_t'(a) D(\theta^*_a || \theta_a) \approx \log
T \] is met for all particles $\theta \in S_{\mathsf{a}_k}'$ at time
$t$, with a slight modification in that the play count $N_t'(a)$ is
``frozen'' to $N_{\tau_m}(a_m)$ if action $a$ has been eliminated at
an earlier time $\tau_m \leq t$, and the introduction of the factor
$\frac{1+\epsilon}{1-\epsilon}$ to the right hand side.

In case more than one suboptimal action is eliminated for the first
time at $\tau_k$, we use a fixed tie-breaking rule in $\mathcal{A}$ to
resolve the tie. We then put 
\[\mathcal{A}_{k} \bydef \mathcal{A}_{k-1}
\cup \{\mathsf{a}_k\}.\] 
Thus, $\tau_0 \leq \tau_1 \leq \ldots \leq \tau_{|\mathcal{A}|-1} \leq
T$, and $\mathcal{A}_0 \subseteq \mathcal{A}_1 \subseteq \ldots
\subseteq \mathcal{A}_{|\mathcal{A}|-1} = \mathcal{A}$. \\

For each action $a \neq a^*$, by definition, there exists a unique
$\tau_k$ for which $a$ is first eliminated at $\tau_k$, i.e.,
$\mathcal{A}_k \setminus \mathcal{A}_{k-1} = a$. Let $\tau(a) \bydef
\tau_k$. \\

The following lemma states that after an action $a$ is eliminated, the
algorithm does not sample from $S_a'$ more than a constant number of
times.
\begin{lemma}
  \label{lem:notmuchSa'}
  If $\log T \geq \lambda |\mathcal{A}|$, then
  \[ \prob{G, \forall k \; \sum_{t=\tau_k+1}^T \mathbf{1}\{\theta_t
    \in S_{\mathsf{a}_k}'\} > \frac{|\mathcal{A}|}{\delta
      \pi(\theta^*)}} \leq \delta.\]
\end{lemma}
\begin{proof}
  Observe that under $G$, whenever $T \geq t > \tau_k$, every $\theta
  \in S_{\mathsf{a}_k}'$ satisfies
  \begin{align*}
    W_t(\theta) &\leq \exp\left(-\sum_{a \in \mathcal{A}}
      b_{\theta,a}(N_t'(a))\right) \\
    &\leq \exp\left(-\sum_{a \in \mathcal{A}} \left((1-\epsilon) N_t'(a)
        D(\theta^*_a || \theta_a) - \lambda \right) \right) =
    \exp\left(-(1-\epsilon)\sum_{a \in \mathcal{A}}
      N_t'(a) D(\theta^*_a || \theta_a) + \lambda |\mathcal{A}|  \right) \\
    &\leq \exp\left(-(1-\epsilon)\sum_{m=1}^{k-1}
      N_{\tau_m}'(a_m)D(\theta^*_{a_m} || \theta_{a_m}) -(1-\epsilon)
      \sum_{\substack{a \notin \mathcal{A}_{k-1}}}
      N_t'(a) D(\theta^*_a || \theta_a) + \lambda |\mathcal{A}| \right)\\
    &\leq \exp\left( -(1-\epsilon) \frac{1+\epsilon}{1-\epsilon}\log T + \epsilon \log T\right) = \frac{1}{T}.
\end{align*}
The second inequality above is because the definition of
$b_{\theta,a}(x)$ implies that $\forall x \geq 0 \;
(1-\epsilon)xD(\theta^*_a||\theta_a) - b_{\theta,a}(x) \leq \lambda$. The
penultimate inequality above is due to the fact that for any $m \leq
k$, we have $\tau_m \leq \tau_k \leq t$, implying that $N_t'(a_m) \geq
N_{\tau_m}'(a_m)$. We now estimate
  \begin{align*}
    &\expect{\mathbf{1}\{t > \tau_k\}\mathbf{1}\{\theta_t \in
      S_{\mathsf{a}_k}' \} \given G} = \expect{ \expect{
        \mathbf{1}\{t >
        \tau_k\}\mathbf{1}\{\theta_t \in S_{\mathsf{a}_k}'\}\given G, \mathcal{F}_t } \given G} \\
    &= \expect{\mathbf{1}\{t > \tau_k\} \pi_t(S_{\mathsf{a}_k}') \given G} = \expect{\mathbf{1}\{t > \tau_k\}
      \frac{\int_{S_{\mathsf{a}_k}'} W_t(\theta)\pi(d\theta)
      }{\int_\Theta W_t(\theta)\pi(d\theta)} \given G} \\
    &\leq \expect{\mathbf{1}\{t > \tau_k\}
      \frac{T^{-1}}{\pi(\theta^*)} \given G} \leq
    \frac{T^{-1}}{\pi(\theta^*)},
  \end{align*}
  which implies that 
  \[\expect{ \sum_{t=\tau_k+1}^T \mathbf{1}\{\theta_t \in
    S_{\mathsf{a}_k}'\} \given G} = \sum_{t=1}^T \expect{\mathbf{1}\{t > \tau_k\}\mathbf{1}\{\theta_t \in
    S_{\mathsf{a}_k}' \} \given G} \leq \frac{1}{\pi(\theta^*)}. \]
  Thus, 
  \[ \prob{ \sum_{t=\tau_k+1}^T \mathbf{1}\{\theta_t \in
    S_{\mathsf{a}_k}'\} > \frac{1}{\delta \pi(\theta^*)} \given G}
  \leq \delta.\] Replacing $\delta$ by $\frac{\delta}{|\mathcal{A}|}$
  and taking a union bound over $k = 1, 2, \ldots, |\mathcal{A}|-1$
  proves the lemma.
\end{proof}

Now we bound the number of plays of suboptimal actions under the event
\[ H \bydef G \; \bigcap \left\{\exists a \neq a^* \;
  \sum_{t=1}^\infty \mathbf{1}\{\theta_t \in S_a''\} \leq
  \frac{\alpha |\mathcal{A}|}{\delta} \right\} \bigcap \left\{\forall
  k \; \sum_{t=\tau_k+1}^T \mathbf{1}\{\theta_t \in
  S_{\mathsf{a}_k}'\} \leq \frac{|\mathcal{A}|}{\delta \pi(\theta^*)}
\right\},\] which, according to the results of Theorem
\ref{thm:selfnormalized}, Lemma \ref{lem:notmuchSa''} and Lemma
\ref{lem:notmuchSa'}, occurs with probability at least $1 - (\delta
\sqrt{2} + 2\delta)$. Under the event $H$, we have
\begin{align*}
  \sum_{a \neq a^*} N_T'(a) &= \sum_{k=1}^{|\mathcal{A}|-1}
  N_T'(\mathsf{a}_k) \\
  &= \sum_{k=1}^{|\mathcal{A}|-1} N_{\tau_k}'(\mathsf{a}_k) +
  \sum_{k=1}^{|\mathcal{A}|-1} (N_T'(\mathsf{a}_k) - N_{\tau_k}'(\mathsf{a}_k))  \\
  &= \sum_{k=1}^{|\mathcal{A}|-1} N_{\tau_k}'(\mathsf{a}_k) +
  \sum_{k=1}^{|\mathcal{A}|-1} \sum_{t = \tau_k+1}^T
  \mathbf{1}\{\theta_t \in S_{\mathsf{a}_k}'\} \\
  &\leq \sum_{k=1}^{|\mathcal{A}|-1} N_{\tau_k}'(\mathsf{a}_k) +
  \frac{|\mathcal{A}|^2}{\delta \pi(\theta^*)}. 
\end{align*}
\begin{lemma}
  Under $H$, $\sum_{k=1}^{|\mathcal{A}|-1} N_{\tau_k}'(\mathsf{a}_k)
  \leq \mathsf{C}_T$, where $\mathsf{C}_T$ solves
  \begin{equation}
    \label{eqn:optimization}
    \begin{aligned}
      \mathsf{C}(\log T) \bydef \; &\max && \sum_{k=1}^{|\mathcal{A}|-1} z_k(a_k)\\
      &\text{ s.t.}
      && z_k \in \mathbb{Z}_+^{|\mathcal{A}|-1} \times \{0\}, a_k \in \mathcal{A} \setminus \{a^*\}, 1 \leq k \leq |\mathcal{A}|-1, \\
      &&& z_i \succeq z_k, z_i(a_k) = z_k(a_k), i \geq k, \\
      &&& \forall 1 \leq j,k \leq |\mathcal{A}|-1:\\
      &&& \quad \min_{\theta \in S_{a_k}'} \quad \ip{z_k}{D_\theta} \geq \frac{1+\epsilon}{1-\epsilon}\log T, \\
      &&& \quad \min_{\theta \in S_{a_k}'} \quad \ip{z_k -
        e^{(j)}}{D_\theta} < \frac{1+\epsilon}{1-\epsilon}\log T.
    \end{aligned}
  \end{equation}

\end{lemma}
\begin{proof}
  With regard to the definition of the $\tau_k$ and $\mathsf{a}_k$ in
  (\ref{eqn:deftau}), if we take 
  \[ a_k = \mathsf{a}_k, \quad 1 \leq k \leq |\mathcal{A}|-1, \] and
  \begin{equation*}
    z_k(a) = \left\{ 
      \begin{array}{cc}
        N_{\tau(a)}'(a), &  \tau(a) \leq \tau_k, \\
        N_{\tau_k}'(a),  &  \tau(a) > \tau_k, 
      \end{array}
    \right.
  \end{equation*}
  then it follows, from (\ref{eqn:deftau}), that the $z_k$ and $a_k$
  satisfy all the constraints of the optimization problem
  (\ref{eqn:optimization}). We also have $\sum_{k=1}^{|\mathcal{A}|-1}
  z_k(k) = \sum_{k=1}^{|\mathcal{A}|-1}
  N_{\tau_k}'(\mathsf{a}_k)$. This proves the lemma.
\end{proof}


\section{Proof of Corollary \ref{cor:corr2}}
\label{app:corr2}
  The optimal action (in this case a subset) is $a^* = \{N-M+1,
  \ldots, N\}$. It can be checked that the assumptions
  \ref{ass:finite}-\ref{ass:uba} are verified, thus the bound
  (\ref{eqn:tsregret}) applies and we will be done if we estimate
  $\mathsf{C}(\log T)$. \\

  
  The essence of the proof is to first partition the space of
  suboptimal actions (subsets) according to the least-index basic
  arm that they contain, i.e., for $i = 1, 2, \ldots, N-M$, let
  \[ \mathcal{A}_i \bydef \left\{a \subset [N]: a \neq a^*, \min \{j
    \in a\} = i \right\} \] be all the actions whose least-index (or
  ``weakest'') arm is $i$ \footnote{This covers all of $\mathcal{A}
    \setminus \{a^*\}$ since every suboptimal set must contain a basic
    arm of index $N-M$ or lesser.}. \\

  Take any sequence $\{z_k\}_{k=1}^{|\mathcal{A}|-1}$,
  $\{a_k\}_{k=1}^{|\mathcal{A}|-1}$ feasible for (\ref{eqn:tsregret}).
  Fix $1 \leq i \leq N-M$ and consider the sum $\sum_{k: a_k \in
    \mathcal{A}_i} z_k(a_k)$. We claim that this does not exceed $1 +
  \left(\frac{1+\epsilon}{1-\epsilon}\right)\frac{1}{D(\mu_i ||
    \mu_{N-M+1})} \log T$. If, on the contrary, it does, then put
  $\hat{k} \bydef \max \{k: a_k \in \mathcal{A}_i\}$. Take any model
  $\theta \in S_{a_{\hat{k}}}'$. We must have
  $D(\mu_{a^*}||\theta_{a^*}) = 0$. Since the KL divergence due to
  observing a tuple of $M$ independent rewards is simply the sum of
  the $M$ individual (binary) KL divergences, we get that $\theta_j =
  \mu_j$ for all $j \geq N-M+1$. However, the optimal action for
  $\theta$ is $a_{\hat{k}}$ containing the basic arm $i$. Hence, we
  get that $\theta_i \geq \mu_{N-M+1} \geq \mu_i$, which implies that
  $D(\mu_i||\theta_i) \geq D(\mu_i || \mu_{N-M+1})$. \\

  It now remains to estimate
  \begin{align*}
    \ip{z_{\hat{k}} - e^{(\hat{k})}}{D_{\theta}} &= \sum_{j=1}^N
    \ip{\sum_{a: j \in a} z_{\hat{k}}(a) - \delta_{j \in
        a_{\hat{k}}}}{D(\mu_j || \theta_j)} \\ &\geq \left( \sum_{a: i
      \in a} z_{\hat{k}}(a) - 1 \right) {D(\mu_i || \theta_i)}
    \\ &\geq \left( \sum_{a \in \mathcal{A}_i} z_{\hat{k}}(a) - 1
    \right){D(\mu_i ||\mu_{N-M+1})} \\ &= \left( \sum_{k: a_k \in
      \mathcal{A}_i} z_k(a_k) - 1 \right){D(\mu_i ||\mu_{N-M+1})}
    \\ &> \log T,
  \end{align*}
  by hypothesis. This violates the final inequality of
  (\ref{eqn:tsregret}) and yields the desired contradiction. Since the
  above argument is valid for any $1 \leq i \leq N-M$, summing over
  all such $i$ completes the proof. 

\section{Proof of Proposition \ref{prop:atleastLmain} \& Corollary \ref{cor:corr3}}
\label{app:corr3}
\begin{lemma}
\label{lem:ubCT1}
  Let $T$ be large enough such that $\max_{\theta \in \Theta, a \in
    \mathcal{A}} D(\theta^*_a || \theta_a) \leq
  \frac{1+\epsilon}{1-\epsilon}\log T$. Then, the optimization problem
  (\ref{eqn:tsregret}) admits the following upper bound:
  \begin{equation}
    \label{eqn:ubCT1}
    \begin{aligned}
      \mathsf{C}(\log T) \leq \; & \max && ||z||_1 \\
      &\text{ s.t.}
      && z \in \mathbb{R}^{|\mathcal{A}|-1} \times \{0\}, \\
      &&& a \in \mathcal{A}, a \neq a^*, \\
      &&& \min_{\theta \in S_{a}'} \quad \ip{z}{D_\theta} \leq \frac{2(1+\epsilon)}{1-\epsilon}\log T,  \\
      &&& 0 \leq z({\hat{a}}) \leq \frac{2}{\delta_{\hat{a}}}
      \left(\frac{1+\epsilon}{1-\epsilon}\right) \log T, \quad \forall
      \hat{a} \in \mathcal{A}, \hat{a} \neq a^*.
    \end{aligned}
  \end{equation}
\end{lemma}

\begin{proof}
  Take a feasible solution $\{z_k,a_k\}_{k=1}^{|\mathcal{A}|-1}$ for
  the optimization problem (\ref{eqn:tsregret}). We will show that $z
  = z_{|\mathcal{A}|-1}$ and $a = a_{|\mathcal{A}|-1}$ satisfy the
  constraints (\ref{eqn:ubCT1}) above and yield the same objective
  function value in both optimization problems. 

  First,
  \[ ||z||_1 = \sum_{\hat{a} \in \mathcal{A},\hat{a} \neq a^*}
  z(\hat{a}) = \sum_{k=1}^{|\mathcal{A}|-1} z_{|\mathcal{A}|-1}(a_k) =
  \sum_{k=1}^{|\mathcal{A}|-1} z_{k}(a_k), \] as
  $z_{|\mathcal{A}|-1}(a_k) \geq z_k(a_k)$, for all $k \leq
  |\mathcal{A}|-1$, by (\ref{eqn:tsregret}). This shows that the
  objective functions of both (\ref{eqn:tsregret}) and
  (\ref{eqn:ubCT1}) are equal at $\{z_k,a_k\}_{k=1}^{|\mathcal{A}|-1}$
  and $(z,a)$ respectively. \\
   
  Next, for any $1 \leq j \leq |\mathcal{A}|-1$ and the unit vector
  $e^{(j)}$, we have
  \begin{align*}
    \min_{\theta \in S_{a}'} \quad \ip{z}{D_\theta} &= \min_{\theta
      \in S_{a_k}'} \quad \ip{z_k}{D_\theta} \\
    &\leq \min_{\theta \in S_{a_k}'} \quad \ip{z_k -
      e^{(j)}}{D_\theta} + \max_{\theta \in \Theta, a \in \mathcal{A}}
    D(\theta^*_a || \theta_a) \\
    &\leq \frac{1+\epsilon}{1-\epsilon}\log T +
    \frac{1+\epsilon}{1-\epsilon}\log T =
    \frac{2(1+\epsilon)}{1-\epsilon}\log T.
  \end{align*}
  This shows that the penultimate constraint in (\ref{eqn:ubCT1}) is
  satisfied. For the final constraint in (\ref{eqn:ubCT1}), fix $1
  \leq j \leq |\mathcal{A}|-1$, so that we have
  \begin{align*}
    \delta_{a_j}\cdot z(a_j) = \delta_{a_j}\cdot z_j(a_j) \leq
    \min_{\theta \in S_{a}'} \quad \ip{z_j}{D_\theta} \leq
    \frac{2(1+\epsilon)}{1-\epsilon}\log T,
  \end{align*}
  exactly as in the preceding derivation. This implies that
  $z(\hat{a}) \leq \frac{2}{\delta_{\hat{a}}}
  \left(\frac{1+\epsilon}{1-\epsilon}\right) \log T$ for all $\hat{a}
  \neq a^*$.
\end{proof}

\setcounter{thm}{1}
\begin{proposition}
  Let $T$ be large enough such that $\max_{\theta \in \Theta, a \in
    \mathcal{A}} D(\theta^*_a || \theta_a) \leq
  \frac{1+\epsilon}{1-\epsilon}\log T$. Suppose
 \[\Delta \leq \min_{a \neq a^*} \delta_a  = \min_{a \neq a^*, \theta \in S_a'} D(\theta^*_a || \theta_a).\]
 Suppose also that $L \in \mathbb{Z}^+$ is such that for every $a \neq a^*$ and
 $\theta \in S_a'$,
  \[ |\{\hat{a} \in \mathcal{A}: \hat{a} \neq a^*,
  D(\theta^*_{\hat{a}} || \theta_{\hat{a}}) \geq \Delta \}| \geq L, \]
  i.e., at least $L$ coordinates of $D_\theta$ (excluding the
  $|\mathcal{A}|$-th coordinate $a^*$) are at least $\Delta$. Then,
  \[\mathsf{C}(\log T) \leq \left(\frac{|\mathcal{A}| - L}{\Delta}\right)
  \frac{2(1+\epsilon)}{1-\epsilon} \log T. \]
\end{proposition}
  
\begin{proof}[Proof of Proposition \ref{prop:atleastLmain}]
  Consider a solution $(z, a)$ to a \emph{relaxation} of the
  optimization problem (\ref{eqn:ubCT1}) obtained by replacing
  $\delta_{\hat{a}}$ with $\Delta$ and $D_\theta$ with
  $D_\theta' \bydef \min(D_\theta,\Delta\cdot\mathbf{1}) \preceq
  D_\theta $ \footnote{Here $\mathbf{1}$ represents an all-ones
    vector of dimension $\mathcal{A}$, and the minimum is taken
    coordinatewise. Also, a solution exists since the objective is
    continuous and the feasible region is compact.}. We claim that
  $||z||_1 \equiv \ip{\mathbf{1}}{z}\leq \left(\frac{|\mathcal{A}| -
      L}{\Delta}\right) \chi$ where $\chi \bydef
  \frac{2(1+\epsilon)}{1-\epsilon} \log T$. If not, let $y = \chi
  \left(\frac{1}{\Delta},\ldots,\frac{1}{\Delta},0 \right)$, and
  observe that
  \begin{align*}
    \ip{D_\theta'}{y-z} &= \ip{D_\theta'}{y} - \ip{D_\theta'}{z} \\
    &\geq \chi \cdot L \cdot \Delta \cdot \frac{1}{\Delta} - \chi =
    \chi(L-1).
  \end{align*}
  But then,
  \begin{align*}
    \ip{\mathbf{1}}{y-z} &= \ip{\mathbf{1}}{y} - \ip{\mathbf{1}}{z} \\
    &< \frac{\chi(|\mathcal{A}|-1)}{\Delta} -
    \frac{\chi(|\mathcal{A}|-L)}{\Delta}
    = \frac{\chi(L-1)}{\Delta} \\
    &\leq \frac{\ip{D_\theta'}{y-z}}{\Delta} \\
    &\leq \frac{\ip{\Delta \cdot \mathbf{1}}{y-z}}{\Delta} =
    \ip{\mathbf{1}}{y-z},
  \end{align*}
  since $D_\theta' \preceq \Delta \cdot \mathbf{1}$ by definition and
  $z \preceq y$ by hypothesis. This is a contradiction.
\end{proof}

{\bf Playing Subsets with Max reward:} Let $\beta \in (0,1)$, and
suppose that $\Theta = \{1-\beta^R, 1-\beta^{R-1}, \ldots, 1-\beta^2,
1-\beta\}^N$, for positive integers $R$ and $N$. Consider an $N$ armed
Bernoulli bandit with arm parameters $\mu \in \Theta$. The complex
actions are all size $M$ subsets of the $N$ basic arms, $M \leq
\frac{N-1}{2}$. Let $\mu_{min} \bydef \min_{a \in \mathcal{A}}
\prod_{i \in a} (1-\mu_i)$.


\begin{proof}[Proof of Corollary \ref{cor:corr3}]
  Since the reward from playing a subset $a$ is the maximum
  (equivalently, the Boolean OR) value, the marginal KL divergence
  along action $a$ is simply the Bernoulli KL divergence for the
  product of the parameters: $D(\theta^*_a || \theta_a) = D(\mu_a ||
  \theta_a) = D\left(\prod_{i \in a}(1-\mu_i) || \prod_{i \in a} (1-\theta_i)
  \right)$. \\

 Let us estimate
  \begin{align*}
    \Delta \bydef \min\{D(\mu_a||\theta_a): \theta \in \Theta, a
    \in \mathcal{A}, D(\mu_a||\theta_a) > 0 \}. 
  \end{align*}
  If $\mu_i = 1-\beta^{r_i}$ and $\theta_i = 1-\theta^{s_i}$ for
  integers $r_i, s_i$, $i = 1, 2, \ldots, N$, then Pinsker's
  inequality yields
  \begin{align*}
    D(\mu_a || \theta_a) &\geq \frac{2}{\log 2} \left(\prod_{i \in
        a}(1-\mu_i) - \prod_{i \in a} (1-\theta_i) \right)^2 \\
    &= \frac{2}{\log 2} \left(\beta^{\sum_{i \in a}r_i} -
      \beta^{\sum_{i \in a} s_i} \right)^2 \\
    &= \frac{2}{\log 2} \beta^{2\sum_{i \in a}r_i} \left(1 -
      \beta^{\sum_{i \in a} s_i- \sum_{i \in a}r_i} \right)^2. 
  \end{align*}
  $D(\mu_a||\theta_a) > 0$ if and only if $|\sum_{i \in a} s_i-
  \sum_{i \in a}r_i| \geq 1$. This implies, together with the above,
  that
  \[ \Delta \geq \frac{2 \mu^2_{\min}(1-\beta)}{\log 2}.\]

  Next, we claim that for any $\mu \neq \theta \in \Theta$, $D(\mu_a
  || \theta_a) > 0$ for at least $L = {N-1 \choose M-1} - 1$ size $M$
  subsets/actions $a$. This is because if otherwise, then $\sum_{i \in
    a} r_i = \sum_{i \in a} s_i$ for at least ${N \choose M} - L = {N
    \choose M} - {N-1 \choose M-1}+ 1 = {N-1 \choose M} + 1$ subsets
  $a$. However, a combinatorial result \cite{AhlAydKha03:weight}
  states that the maximum number of weight $M$ vertices of the $N$
  dimensional hypercube (in our case, a size $M$ subset corresponds to
  a weight $M$ vertex) that \emph{do not} span $N$ dimensions is ${N-1
    \choose M}$. This forces $r_i = r_i$ for all $i \in [N]$ and hence
  $\mu = \theta$, a contradiction. \\

  Now, we can apply Proposition \ref{prop:atleastLmain} with $\Delta$ and $L$ as
  above. This gives us that for $T$ large enough, the total number of
  arm plays is bounded above, with probability at least $1-\delta$, by
  \begin{align*}
    \mathsf{B}_3 +& (\log 2)\left(
      \frac{1+\epsilon}{1-\epsilon}\right) \left[ {N \choose M} - {N-1
        \choose M-1} + 1 \right] \frac{ \log
      T}{\mu^2_{\min}(1-\beta)} \\
    &= \mathsf{B}_3 + (\log 2)\left(
      \frac{1+\epsilon}{1-\epsilon}\right) \left[{N-1 \choose M} + 1
    \right] \frac{ \log T}{\mu^2_{\min}(1-\beta)}.
\end{align*}

\end{proof}

\end{document}